\documentclass[draftcls,onecolumn,11pt]{IEEEtran}

\usepackage{amsmath,amstext,amssymb,epsf}
\usepackage{fullpage,latexsym}
\usepackage{epsfig}
\usepackage{setspace}
\numberwithin{equation}{section} 

 \newtheorem{lemma}{Lemma}[section]
 \newtheorem{theorem}[lemma]{Theorem}

 \newtheorem{claim}[lemma]{Claim}
 \newtheorem{corollary}[lemma]{Corollary}
 
 \newtheorem{definition}[lemma]{Definition}
 \newtheorem{rem}[lemma]{Remark}

 \newtheorem{ex}[lemma]{Example}

\begin{document}

\title{Information Distance in Multiples}
\author{Paul M.B. Vit\'{a}nyi
\thanks{
Paul Vit\'{a}nyi is with the Centre for Mathematics and Computer Science (CWI),
and the University of Amsterdam.
Address:
CWI, Science Park 123,
1098XG Amsterdam, The Netherlands.
Email: {\tt Paul.Vitanyi@cwi.nl}.
He was supported in part by the
the ESF QiT Programmme, the EU NoE PASCAL II, and
the Netherlands BSIK/BRICKS project.
}}


\maketitle

\begin{abstract}
Information distance is a parameter-free similarity measure
based on compression, used in pattern recognition, data mining,
phylogeny, clustering, and classification.
The notion of information distance
is extended from pairs to multiples (finite lists). 
We study maximal overlap, metricity, universality, 
minimal overlap, additivity, and normalized
information distance in multiples.
We use the theoretical notion of Kolmogorov complexity which
for practical purposes is approximated by the length
of the compressed version of the file involved, using
a real-world compression program.

\end{abstract}

\section{Introduction}
\label{sect.intro}

In pattern recognition, learning, and data mining one obtains
information from objects containing information. This involves
an objective definition of the information
in a single object, the information to go from one object to
another object in a pair of objects, the information to go from
one object to any other object in a multiple of objects,
and the shared information between objects, 
\cite{TKS02}.  

The classical notion of
Kolmogorov complexity \cite{Ko65} is an objective measure 
for the information in an 
a {\em single} object, and information distance measures the information 
between a {\em pair} of objects \cite{BGLVZ98}. This last notion has
spawned research in the theoretical direction, among others 
\cite{CMRSV02,VV02,Vy02,Vy03,MV01,SV02}. 
Research in the practical direction has focused on 
the {\em normalized} information distance, the similarity metric, which arises
by normalizing the information distance in a proper manner and 
approximating the Kolmogorov complexity through real-world
compressors \cite{Li03,CVW03,CV04,CV07},
This normalized information distance is
a parameter-free, feature-free, and alignment-free
similarity measure  that
has had great
impact in applications.
A variant of this compression distance has been tested
on all time sequence databases used in the last decade in the major
data mining conferences (sigkdd, sigmod, icdm, icde, ssdb, vldb, pkdd, pakdd)
\cite{KLRWHL07}. The conclusion is that the method is competitive with all 51 
other methods used and superior in 
heterogenous data clustering and anomaly detection. In \cite{CAO07a}
it was shown that the method is resistant to noise. This theory 
has found many applications in pattern recognition, phylogeny, clustering,
and classification. For objects that are represented as
computer files such applications range from weather forecasting, software,
earthquake prediction, music, literature, ocr, bioinformatics, to internet
\cite{AS05,BCL02,CFLMS04,CBVA06,CV04,CV07,ERVR05,KJ04,KKKP06,KSAG05,KP04,LBCKKZ01,OS03,NPAetal08,NPLetal08,We07}. 
For objects that are only represented by name, 
or objects that are abstract like `red,' `Einstein,' 
`three,' the normalized information distance uses background information
provided by Google, or any search engine that produces aggregate page 
counts. It discovers the `meaning' of words and phrases in the sense
of producing a relative semantics. 
Applications run from ontology, semantics, tourism on the web, taxonomy,
multilingual questions, to
question-answer systems 
\cite{GKAH07,HDL07,ZHZL07,XWF08,WLB08,ZWC07,HH08,FH07,GK07}.
For more references on either subject
see the textbook \cite{LiVi08} or Google Scholar for references
to \cite{Li03,CV04,CV07}. 

However, in many applications we are interested in 
shared information between many objects
instead of just a pair of objects. For example, in customer reviews of
gadgets, in blogs about public happenings,
in newspaper articles about the same occurrence, we are interested
in the most comprehensive one or the most specialized one.  
Thus, we want to extend the information distance 
measure from pairs to multiples.

\subsection{Related Work}\label{sect.relwork}

In \cite{Li08} the notion is introduced of the information required to
go from any object in a multiple of objects to any other object in the
multiple. This
is applied to extracting the essence from, for example, a finite list
of internet news items,  
reviews of electronic cameras, tv's, and so on, in a way
that works better than other methods. 
Let $X$ denote a finite list of 
$m$ finite binary strings defined by $X=(x_1, \ldots, x_m)$, the constituting
strings ordered length-increasing lexicographic.
We use lists and not sets, since if $X$ is a set we cannot express
simply the distance from a string to itself or between strings that
are all equal.
Let $U$ be the reference universal Turing machine, for convenience the prefix
one as in Section~\ref{sect.prel}.
Given the string $x_i$ we define the information distance to any string in $X$
by $E_{\max} (X) = \min \{|p|: U(x_i,p,j)=x_j$ for all $x_i,x_j \in X$\}.
It is shown in \cite{Li08}, Theorem 2, that 
\begin{equation}\label{eq.li08}
E_{\max}(X)=
\max_{x:x \in X} K(X|x),
\end{equation}
up to a logarithmic additive term.
Define $E_{\min}(X) = \min_{x:x \in X} K(X|x)$.
Theorem 3 in \cite{Li08} states that for every list $X= (x_1, \ldots , x_m)$ we have
\begin{equation}\label{eq.li083}
E_{\min}(X) \leq E_{\max}(X) \leq
\min_{i: 1 \leq i \leq m} 
\sum_{x_i,x_k \in X \; \& \; k \neq i} E_{\max} (x_i,x_k),  
\end{equation}
up to a logarithmic additive term. This is not a corollary of 
\eqref{eq.li08} as stated in \cite{Li08}, but
both inequalities follow from the definitions. 
The lefthand side is interpreted as
the program length of the 
``most comprehensive object that contains the most information
about all the others [all elements of $X$],'' 
and the righthand side is interpreted as the program length
of the ``most specialized object that
is similar to all the others.''
The paper \cite{Li08} develops the stated results and applications.
It does not develop the theory in any detail. That is the
purpose of the present paper.

\subsection{Results}

Information distance for multiples, that is, finite lists, appears 
both practically and theoretically
promising. 
In all cases below the results imply the corresponding 
ones for the pairwise 
information distance defined as follows. 
The information distance in \cite{BGLVZ98}
between strings $x_1$ and $x_2$ is $E_{\max}(x_1,x_2) 
= \max\{K(x_1|x_2),K(x_2|x_1)\}$.
In the current paper $E_{\max}(X) = \max_{x:x \in X} K(X|x)$.
These two definitions coincide for $|X|=2$ since $K(x,y|x)=K(y|x)$ up
to an additive constant term.
We investigate 
the maximal overlap of information (Theorem~\ref{theo.maxo})
which for $|X|=2$ specializes to Theorem 3.4 in \cite{BGLVZ98},
Corollary~\ref{cor.1} shows \eqref{eq.li08} and Corollary~\ref{cor.2}
shows that the lefthand side of \eqref{eq.li083} can be taken to
correspond to a single program embodying the ``most comprehensive object
that contains the most information about all the others'' as stated but not
argued or proved in \cite{Li08};
metricity (Theorem~\ref{theo.metrics}) 
and universality (Theorem~\ref{t:optimal.cognitive.dist})
which for $|XY|=2$ (for metricity) and $|X|=2$ (for universality) specialize to
Theorem 4.2 in \cite{BGLVZ98}; 
additivity (Theorem~\ref{theo.additive}); 
minimum overlap of information (Theorem~\ref{thm-muchnik})
which for $|X|=2$ specializes to Theorem 8.3.7 in \cite{Mu02};
and the nonmetricity of 
normalized information distance 
for lists of more than two elements and certain proposals
of the normalizing factor
(Section~\ref{sect.nid}). 
In contrast, for 
lists of two elements we can normalize the
information distance as in Lemma V.4 and Theorem V.7 of
\cite{Li03}.
The definitions are of necessity new as are the proof ideas.
Remarkably, the new notation and proofs for the general case
are simpler than the mentioned existing proofs for the particular case
of pairwise information distance.

\section{Preliminaries}\label{sect.prel}

{\bf Kolmogorov complexity:} This is the information in a single object
\cite{Ko65}.
The notion has been the subject of a plethora of papers. 
Informally, the Kolmogorov complexity of a finite binary string
is the length of the shortest string from which the original
can be losslessly reconstructed by an effective
general-purpose computer such as a particular universal Turing machine.
Hence it constitutes a lower bound on how far a
lossless compression program can compress.
For technical reasons we choose Turing machines with a separate 
read-only input tape, that is scanned from left to right without backing up, 
a separate work tape on which the computation takes place, 
and a separate output tape. Upon halting, the initial segment $p$ 
of the input that has been scanned is called the input ``program'' 
and the contents of the output tape is called the ``output.'' 
By construction, the set of halting programs is prefix free. 
We call $U$ the reference universal prefix Turing machine.
This leads to the definition of ``prefix Kolmogorov complexity''
which we shall designate simply as ``Kolmogorov complexity.''

Formally, the {\em conditional Kolmogorov complexity}
$K(x|y)$ is the length of the shortest input $z$
such that the reference universal prefix Turing machine $U$ on input $z$ with
auxiliary information $y$ outputs $x$. The
{\em unconditional Kolmogorov complexity} $K(x)$ is defined by
$K(x|\epsilon)$ where $\epsilon$ is the empty string (of length 0).
In these definitions both $x$ and $y$ can consist of a nonempty finite lists
of finite binary strings. For more details and theorems that are used 
in the present work see Appendix~\ref{sect.kolm}.

{\bf Lists:} A {\em list} is
a multiple $X=(x_1 , \ldots ,x_m)$ of $m < \infty$ 
finite binary strings in
length-increasing lexicographic order. If $X$ is a list, then some
or all of its elements may be equal. Thus, a list is not a set but an ordered
{\em bag} of elements. With some abuse of the common
set-membership notation we write $x_i \in X$
for every $i$ ($0 \leq i \leq m$) to mean that ``$x_i$ is an element
of list $X$.'' The conditional prefix Kolmogorov complexity $K(X|x)$ of a list
$X$ given an element $x$ is the length of a shortest program $p$
for the reference universal Turing machine that with input $x$
outputs the list $X$. The prefix Kolmogorov complexity $K(X)$ of a list
$X$ is defined by $K(X| \epsilon )$. 
One can also put lists in the conditional such as $K(x|X)$ or 
$K(X|Y)$.
We will use the straightforward laws $K(\cdot|X,x)=K(\cdot|X)$
and $K(X|x)=K(X'|x)$ up to an additive constant term, for $x \in X$
and $X'$ equals the list $X$ with the element $x$ deleted.

{\bf Information Distance:}
To obtain the {\em pairwise information distance} in \cite{BGLVZ98} 
we take $X=\{x_1,x_2\}$ in \eqref{eq.li08}. Then \eqref{eq.li08} is equivalent
to $E_{\max}(x_1,x_2) = \max\{K(x_1|x_2), K(x_2|x_1) \}$.

\section{Maximal Overlap}

We use the notation and terminology of Section~\ref{sect.relwork}.
Define
$k_1 = E_{\min}(X)$, $k_2=E_{\max}(X)$,
and $l=k_2-k_1$.
We prove a {\em maximal overlap} theorem: the information needed to go
from any $x_i$ to any $x_k$ in $X$ can be divided in two parts: a single string
of length $k_1$ and a string
$r$ of length $l$ (possibly depending on $x_i$),
everything 
up to an additive logarithmic term. 

\begin{theorem}\label{theo.maxo}
A single program
of length $k_1 +K(m,k_1,k_2)+\log m + O(1)$ bits concatenated
with a string of $l$
bits, possibly depending on $i$, suffice to
find $X$ from $x_i$
for every $x_i \in X$.
To find an arbitrary element $x_k \in X$ from $x_i$
it suffices to concatenate at most another $\log m$ bits, possibly depending
on $i$ and $k$.
\end{theorem}
\begin{proof}
Enumerate the finite binary strings lexicographic length-increasing as
$s_1,s_2, \ldots .$ Let $G=(V,E)$ be a graph defined as follows.
Let $A$ be the set of finite binary strings and 
$B$ the set of vectors of strings in $A$ 
defined by $v=(s_1, \ldots , s_m)$ such
that 
\begin{eqnarray*}
&&\min_{j: 1 \leq j \leq m} \{ K(s_1 , \ldots , s_m|s_j) \} 
\leq k_1, 
\\&& \max_{j:  1 \leq j \leq m} \{ K(s_1 , \ldots , s_m|s_j) \} 
\leq k_2. 
\end{eqnarray*}
Given $k_1$ and $k_2$ the set $B$ can be enumerated.
Define $V= A \bigcup B$.
Define $E$ by length-increasing lexicographic enumerating
$A \times B$ and put $(rs,v) \in E$ with $rs \in A$  and 
$v=(s_1 , \ldots , s_m) \in B$ if $s = s_j$ for some $j$ 
($1 \leq j \leq m$), where
$r$ is chosen as follows. It is the 
$\lceil i/2^{k_1} \rceil$th string of length 
$l$ where $i$ is the number of times we have used $s \in A$.
So the first $2^{k_1}$ times we choose an 
edge $(\cdot s, \cdot)$ we use $0^l$, the next
$2^{k_1}$ we use $0^{l-1}1$, and so on.
In this way, $i \leq 2^{k_2}$
so that $i/2^{k_1} \leq 2^l$. By adding $r$ to $s$ we take care 
that the degree of $rs$ is at most $2^{k_1}$ and not at most $2^{k_2}$
as it could be without the prefix $r$. The degree of a node $v \in B$
is trivially $m$.

In addition, we enumerate $B$ length-increasing lexicographic
and `color' everyone of the $m$ 
edges incident  
with an enumerated vector $v \in B$
with the same binary string $c$ of length $k_1+\log m$.
If $v=(s_1, \ldots ,s_m)$ and $v$ is connected by edges to
nodes $r_1s_1, \ldots , r_ms_m$, then choose
$c$ as the minimum color not yet appearing on any edge incident with
any $r_js_j$  ($1 \leq j \leq m$). 
Since the degree of every node $rs \in A$
is bounded by $2^{k_1}$ and hence the colors already used for edges incident
on nodes $r_1s_1, \ldots , r_ms_m$ number at most 
$\sum_{1 \leq j \leq m} (2^{k_1} -1) = m2^{k_1} - m$,
a color is always available.

Knowing $m,k_1,k_2$ one can reconstruct
$G$ and color its edges.
Given an element $x$ from the list $X$, and knowing the appropriate 
string $r$ of length $l$ and the color $c$ of the edge $(rx,X)$,
we can find $X$.
Hence a single program, say $p$,
of length $k_1 +K(m,k_1,k_2)+\log m +O(1)$ bits suffices to
find $X$ from $rx$ for any $x \in X$ and with $|r|=l$. 
An additional $\log m$
bits suffice to select any element of $X$.  
Taking these $\log m$ bits so that they encode the difference
from $i$ to $k \bmod m$ we can compute from every $x_i \in X$ to 
every $x_k \in X$ and
vice versa with the same program $p$ of length 
$k_1 +K(m,k_1,k_2)+\log m +O(1)$ concatenated with
a string $r$ of length $l$ 
and a string of length $\log m$, both possibly depending
on $i$ and $k$. Since we know $m,k_1,k_2$ from the fixed program $p$,
where they are encoded as a self-delimiting prefix of length $K(m,k_1,k_2)$
say, we can
concatenate these strings without separation markers and reconstruct them.
\end{proof}

\begin{corollary}\label{cor.1}
\rm
Since $k_1+l=k_2$,
the theorem implies \eqref{eq.li08}, that is, Theorem 2 of \cite{Li08}.
\end{corollary}
It is not a priori clear that $E_{\min}(X)$ in the lefthand
side of \eqref{eq.li083} corresponds to a {\em single} program that
represents the information overlap of every shortest program going from
any $x_i$ to the list $X$. This seems in fact assumed
in \cite{Li08} where $E_{\min}(X)$ is interpreted as
the [Kolmogorov complexity of] ``the most comprehensive
object that contains the most information about all the others.''
In fact, for every $x_i \in X$
we can choose a
shortest program going from $x_i$ to the list $X$ so that these 
programs have pairwise no information
overlap at all (Theorem~\ref{thm-muchnik}). But here we have proved: 
\begin{corollary}\label{cor.2}
\rm
The quantity $E_{\min}(X)$ corresponds to a {\em single} shortest
program that represents the maximum overlap of information
of all programs going from $x_i$ to the list $X$ for any
$x_i \in X$.
\end{corollary}

\section{Metricity}

We consider nonempty finite lists of finite binary strings, each list
ordered length-increasing lexicographic. Let ${\cal X}$
be the set of such ordered nonempty finite lists of finite binary strings. 
A {\em distance function} $d$
on ${\cal X}$ is defined by $d:{\cal X} \rightarrow {\cal R}^+$ where 
${\cal R}^+$ is the set of nonnegative real numbers. 
Define $W=UV$ if
$W$ is a list of the elements of the lists $U$ and $V$ and
the elements of $W$ are ordered length-increasing lexicographical.
A distance function
$d$ is a {\em metric} if $X,Y,Z \neq \emptyset$ and 
\begin{enumerate}
\item
{\em Positive definiteness}: $d(X)=0$ if all elements of $X$ are equal
and $d(X) > 0$ otherwise.
\item
{\em Symmetry}: $d(X)$ is invariant
under all permutations of $X$.
\item
{\em Triangle inequality}: $d(XY) \leq d(XZ)+d(ZY)$.
\end{enumerate}
\begin{theorem}\label{theo.metrics}
The information distance for lists, $E_{\max}$, is a metric 
where the (in)equalities hold up to a $O(\log K)$ additive term. 
Here $K$ is the largest quantity
involved in the metric (in)equalities.
\end{theorem}
\begin{proof}
It is clear that 
$E_{\max} (X)$ satisfies positive definiteness and symmetry up to
an $O( \log K)$ additive term where $K= K(X)$. It
remains to show the triangle inequality.
\begin{claim}
Let $X,Y,Z$ be three nonempty finite lists of finite binary strings 
and $K=\max\{K(X),K(Y),K(Z) \}$. Then,
$E_{\max} (XY) \leq E_{\max} (XZ) + E_{\max} (ZY)$ up to an $O(\log K)$
additive term.
\end{claim}
\begin{proof}
By Theorem~\ref{theo.maxo}, 
\begin{eqnarray*}
E_{\max} (XY) & = & \max_{x:x \in XY} K(XY|x) = K(XY|x_{XY}),
\\ E_{\max} (XZ) & = & \max_{x:x \in XZ} K(XZ|x) = K(XZ|x_{XZ}),
\\ E_{\max} (ZY) & = & \max_{x:x \in ZY} K(ZY|x) = K(ZY|x_{ZY}), 
\end{eqnarray*}
equalities up to a $O(\log K)$ additive term.
Here $x_{XY}, x_{XZ}, x_{ZY}$ are the elements for which the maximum
is reached for the respective $E_{\max}$'s.


Assume that $x_{XY} \in X$,
the case $x_{XY} \in Y$ being symmetrical.
Let $z$ be some element of $Z$.
Then,
\begin{eqnarray*}
K(XY|x_{XY}) 
& \leq & K(XYZ|x_{XY})
\\& \leq & K(XZ|x_{XY})+K(Y|XZ, x_{XY})
\\& \leq & K(XZ|x_{XZ})+K(Y|XZ,z)
\\& \leq & K(XZ|x_{XZ})+K(ZY| x_{ZY}) .
\end{eqnarray*}
The first inequality follows from the general $K(u) \leq K(u,v)$,
the second inequality by the obvious subadditive property of $K( \cdot)$,
the third inequality since in the first term $x_{XY} \in XZ$ and
the $\max_{x: x \in XZ} \{K(XZ|x)\}$ is reached for
$x=x_{XZ}$
and in the second term
both $x_{XY} \in X$ and for $z$ take any element from $Z$,
and the
fourth inequality follows by in the second term 
dropping $X$ from the conditional
and moving $Z$ from the conditional to the main argument and observing
that both $z \in ZY$ and the $\max_{x: x \in ZY}\{K(ZY|x)\}$ is reached for
$x=x_{ZY}$.
The theorem follows with
(in)equalities up to an $O(\log K)$ additive term.
\end{proof}
\end{proof}

\section{Universality}

Let $ X \in {\cal X}$.
A priori we allow asymmetric distances. 
We would like to
exclude degenerate distance measures such as $D(X) = 1$ for all $X$.
For each $d$, we want only finitely many lists $X$ 
such that $D(X)\leq d$.
Exactly how fast we want the number of lists we admit 
to go to $\infty$ is not important; it is only a matter of scaling.
For every distance $D$
we require the following {\em density condition}
for every $x \in \{0,1\}^*$:
\begin{equation}\label{eq.density}
 \sum_{X: x \in X \; \& \; D(X) > 0} 2^{-D(X)} \leq  1. 
\end{equation}
Thus, for the density
condition on $D$ we consider only lists $X$ with $|X| \geq 2$ and not
all elements of $X$ are equal. Moreover,
we consider only distances that are computable in some broad sense.
\begin{definition}
\rm
An {\em admissible list distance} $D(X)$
is a total, possibly asymmetric, function from 
${\cal X}$ to the nonnegative real numbers that
is 0 if all elements of $X$ are equal, and greater than $0$ otherwise
(up to an additive
$\log K$ additive term with $K=K(X)$),
is upper semicomputable, and satisfies
the density requirement in \eqref{eq.density}.
\end{definition}

\begin{theorem}\label{t:optimal.cognitive.dist}
The list information distance $E_{\max}(X)$ is admissible
and it is minimal in the sense that for every admissible
list distance function $D(X)$ we have
$
E_{\max} (X) \leq  D(X) 
$
up to an additive constant term.
\end{theorem}
\begin{proof}
It is straightforward
that $E_{\max}(X)$ is a total real-valued function, is 0 only if all elements
of $X$ are equal and unequal 0 otherwise (up to an 
$O(\log K)$ additive term with $K=K(X)$), and is upper semicomputable.
We verify the density requirement of \eqref{eq.density}.
For every $x \in \{0,1\}^*$,
consider lists $X$ of at least two elements not all 
equal and $x \in X$. Define functions
 $f_{x}(X) = 2^{-K(X|x)}$. 
Then,
$f_{x} (X)\geq 2^{-E_{\max}(X)}$.
It is easy to see that for every $x \in \{0,1\}^*$,
\[
\sum_{X: x \in X \; \& \; E_{\max} (X) > 0}   2^{-E_{\max}(X)} \leq
\sum_{X: x \in X \; \& \; E_{\max} (X) > 0}  f_{x}(X)
= \sum_{X: x \in X \; \& \; E_{\max} (X) > 0} 2^{-K(X|x)} 
\leq \sum 2^{-l(p)} 
\]
where the righthand sum is taken over
all programs $p$ for which the reference prefix machine $U$, given
$x$, computes a finite list $X$ of at least two
elements not all equal and such that $x \in X$. 
This sum is the probability that $U$, 
given $x$, computes such a list $X$
from a program $p$ generated bit by bit uniformly at random.
Therefore, the righthand sum is at most 1, and
$E_{\max}(X)$ satisfies the density
requirement \eqref{eq.density}.

We prove minimality. 
Fix any $x \in \{0,1\}^*$.
Since 
$D$ is upper semicomputable, the function $f$
defined by $f(X,x)=2^{-D(X)}$
for $X$ satisfying $x \in X$ and $D(X)>0$, and 0 otherwise, 
is lower semicomputable. Since
$\sum_{X: x \in X \; \& \;D(X)>0} 2^{-D(X)} \leq 1$, we have
$\sum_{X} f(X,x) \leq 1$ 
for every $x$. 
Note that given $D$ we can compute $f$, and hence
$K(f) \leq K(D)+O(1)$.
By the conditional version of \eqref{eq.dominate}
in \cite{LiVi08} Theorem 4.3.2, 
we have 
$c_{D} {\bf m} (X|x \in X) \geq f(X,x)$ 
with $c_{D}=2^{K(f)}=2^{K(D)+O(1)}$, that is, 
$c_D$ is a positive constant
depending on $D$ only.
By the conditional version of \eqref{eq.coding}
in \cite{LiVi08} 
Theorem 4.3.4, 
we have for every $x \in X$ that
$ \log 1/{\bf m} (X|x) = K(X|x)+O(1)$.
Hence, for every $x \in X$ we have
 $\log 1/f(X,x) \leq K(X|x) + \log 1/c_{D} +O(1)$.
Altogether, for every 
admissible distance $D$ and every $x \in \{0,1\}^*$,
and every list $X$ satisfying  $x \in X$,
there is a constant $c_{D}$ such that 
$D(X) \leq K(X|x) + \log 1/c_{D}+O(1)$. Hence,
$D(X) \leq E_{\max}(X) + \log 1/c_{D} +O(1)$.
\end{proof}

\section{Additivity}

\begin{theorem}\label{theo.additive}
$E_{\max}$ is not subadditive: neither
$E_{\max}(X) +E_{\max}(Y) \leq E_{\max} (XY)$
nor
$E_{\max}(X) +E_{\max}(Y) \geq E_{\max} (XY)$,
the (in)equalities up to logarithmic additive terms,
holds for all lists $X,Y$.
\end{theorem}
\begin{proof}
Below, all (in)equalities are taken up to logarithmic additive terms.
Let $x,y$ be strings of length $n$,
$X=(\epsilon, x)$ and $Y=(\epsilon, y)$ with $\epsilon$ denoting
the empty word. Then $E_{\max} (XY) =
E_{\max}(\epsilon, \epsilon, x,y)$, $E_{\max} (X) = K(x)$, and
$E_{\max} (Y) = K(y)$. If $x=y$ and $K(x)=n$,
 then $E_{\max} (XY)=E_{\max} (X)
= E_{\max} (Y)=n$. Hence, $E_{\max} (XY) < E_{\max} (X)+E_{\max} (Y)$.

Let $x,y$ be strings of length $n$ such that 
$K(x|y)=K(y|x)=n$, $K(x),K(y) \geq n$, 
$X=(x,x)$, and $Y=(y,y)$.  Then $E_{\max} (XY) =
E_{\max}(x,x,y,y) = \max \{K(x|y),K(y|x)\}=n$, $E_{\max} (X) = 0$, and
$E_{\max} (Y) = 0$. 
Hence, $E_{\max} (XY) > E_{\max} (X)+E_{\max} (Y)$.
\end{proof}

Let $X=(x)$ and $Y=(y)$.
Note that subadditivity holds for lists of singleton elements since
$E_{\max} (x,y) = \max\{K(x|y),K(y|x)\} \leq K(x)+K(y)$, where the
equality holds up to an additive $O(\log \{K(x|y),K(y|x)\})$ term
and the inequality holds up to an additive constant term..

\section{Minimal Overlap}

Let $X= (x_1, \ldots ,x_m)$ and
$p_i$ be a shortest program converting $x_i$ to $X$ ($1 \leq i \leq m$).
Naively we expect that the shortest program that
that maps $x_i$ to $X$ contains the information about $X$ that
is lacking in $x_i$. However, this is too simple, because different
short programs mapping $x_i$ to $X$ may have different properties.

For example, suppose $X=\{x,y\}$  and both elements are strings
of length $n$ with $K(x|y), K(y|x) \geq n$.
Let $p$ be a program that ignores the input and prints $x$.
Let $q$ be a program
such that $y \oplus q=x$ (that is, $q=x \oplus y$),
where $\oplus$ denotes bitwise addition modulo 2.
Then, the programs $p$ and $q$ have nothing in common.

Now let $x$ and $y$ be arbitrary strings
of length at most $n$. Muchnik, Theorem 8.3.7 in \cite{Mu02}, shows that
there exists
a shortest program $p$ that converts $y$ to $x$ (that is, $|p|=K(x|y)$ 
and $K(x|p,y)=O(\log n)$),
such that $p$ is simple
with respect to $x$  and therefore depends little on the origin $y$,
that is, $K(p|x) = O(\log n$).
This is
a fundamental coding property for individual strings that parallels
related results about random variables
known as
the Slepian--Wolf and Csisz\'ar--K\"orner--Marton theorems \cite{CT91}.
\begin{theorem}\label{thm-muchnik}
Let $X= (x_1, \ldots ,x_m )$ be a list of binary strings 
of length at most $n$.
For every $x_i \in X$ there
exists a string $p_i$ of length $K(X|x_i)$ such that
$K(p_i|X) = O( \log mn)$
and $K(X|p_i,x_i) = O(\log mn)$.
\end{theorem}
\begin{proof}
Muchnik's theorem as stated before gives a code $p$ for 
$x$ when $y$ is known. There, we assumed that $x$
and $y$
have length at most $n$. The proof in \cite{Mu02} does
not use any assumption about $y$. Hence we can extend
the result to information distance in finite lists as follows.
Suppose we encode the constituent list elements
of $X$ self-delimitingly in altogether $mn+O(\log mn)$ bits (now
$X$ takes the position of $x$ and we consider strings of
length at most $mn+O(\log mn)$). 
Substitute $y$ by $x_i$ for some
$i$ ($1 \leq i \leq m$).
Then the theorem above follows straightforwardly 
from Muchnik's original theorem about two strings of length at most $n$.
\end{proof}

The code $p_i$ is not uniquely determined. For example, let $X=(x,y)$
and $z$ be a string such that $|x|=|y|=|z|=n$, 
$K(y|z)=K(z|y) \geq n$, and
and $x=y \oplus z$. Then, both $z$ and $y \oplus z$ 
can be used for $p$ with $K(p|X)=O(1)$ and $K(X|p,y)=O(1)$.
But $z$ and $y \oplus z$ have no mutual information at all.
\begin{corollary}\label{th.vere}
\rm
Let $X=(x_1, \ldots , x_m)$.
For every string $x_i$ there is a program $p_i$
such that $U(p_i,x_i)=X$ ($1 \leq i \leq m$), 
where $|p_i|=K(X|x_i)$, 
and $K(p_i)-K(p_i|p_j)=K(p_j)-K(p_j|p_i)=0$ $(i \neq j)$, and the last four
equalities hold up to an additive $O(\log K(X))$ term.
\end{corollary}

\section{Normalized List Information Distance}\label{sect.nid}

The quantitative difference in a certain feature between many objects
can be considered as an admissible distance, provided it is 
upper semicomputable and satisfies the density condition 
\eqref{eq.density}.
Theorem~\ref{t:optimal.cognitive.dist} shows that
$E_{\max}$ is universal in
that among all admissible list distances in that it is always least.
That is, it accounts for the dominant feature in which the elements of 
the given list are alike.
Many admissible distances
are absolute, but if we want to express similarity,
then we are more interested in relative ones.
For example, if two strings of $1{,}000{,}000$ bits have information
distance $1{,}000$ bits,
then we are inclined to think that those strings are relatively
similar. But if two strings of $1{,}000$ bits have
information distance $1{,}000$ bits,
then we find them very different.

Therefore, our objective is to normalize the universal information distance
$E_{\max}$ to obtain
a universal similarity distance.
It should give a similarity
with distance 0 when the objects in a list are maximally similar 
(that is, they are equal)
and distance 1 when
they are maximally dissimilar.
Naturally, we desire the normalized version of the universal list information
distance metric to be also a metric. 

For pairs of objects, say $x,y$,
the normalized version $e$ of $E_{\max}$ defined by
\begin{equation}\label{eq.pairs}
e(x,y) = \frac{E_{\max} (x,y)}{\max \{K(x),K(y)\}}
= \frac{\max\{ K(x,y|x),K(x,y|y\}}{\max \{K(x),K(y)\}}
\end{equation}
takes values in $[0,1]$ and is a metric. Several alternatives for the
normalizing factor $1/\max \{K(x),K(y)\}$ do not work. Dividing by
the length, either the sum or the maximum does not satisfy the triangle
property. Dividing by $K(x,y)$ results in 
$e_1(x,y) = E_{\max} (x,y)/K(x,y) = \frac{1}{2}$ for $|x|=|y|=n$
and $K(x|y)=K(y|x) \geq n$ (and hence $K(x),K(y)\geq n$), and this is
improper as $e_1(x,y)$ should be 1 in this case. 
We would like a proposal for a normalization factor for lists
of more than two elements to reduce to that of \eqref{eq.pairs}
for lists restricted to two elements. This leads to the 
proposals below, which turn out to be improper.

As a counterexample to normalization
take the following lists: $X=(x)$, $Y = (y)$, and $Z=(y,y)$.
With $|x|=|y|=n$ and the equalities below up to an $O(\log n)$
additive term we define: 
$K(x)=K(x|y) = K(x,y|y)=K(x,y,y|y) = n$,
$K(y|x) = K(y,y|x)=K(x,y,y|x) = 0.9n$, 
and $K(y) = K(y,y) = 0.9n$. 
Using the symmetry of information \eqref{eq.soi} we have
$K(x,y) = 1.9n$.
Let $U,V,W$ be lists. We show that for the proposals below the
triangle property $e(UV) \leq e(UW)+e(WV)$ is violated.
\begin{itemize}
\item
Consider the normalized list information distance
\begin{equation}\label{eq.prop}
e(V)= \frac{ K(V|x_{V})}{K(V_{\max})}.
\end{equation}
That is, we divide $E_{\max} (V)$ by $K(V_{\max})$
with $V_{\max} = \max_i \{V_i\}$ where the
list $V_i$ equals the list $V$ with the $i$th element deleted 
($1 \leq i \leq |V|$). Then, with equalities holding up
to $O((\log n) /n)$ we have: $e(XY)=K(x,y|y)/K(x) = 1$,
$e(XZ)= E_{\max}(XZ)/K(XZ_{\max}) = K(x,y,y|y)/K(x,y) = \frac{1}{2}$,
and $e(ZY) = E_{\max}(ZY)/K(ZY_{\max}) = K(y,y,y|y)/K(y,y) = 0$.
Hence the triangle inequality does not hold.
\item
Instead of dividing by $K(V_{\max})$ in \eqref{eq.prop}
divide by $K(V')$ where $V'$ equals $V$ with $x_V$ deleted.
The same counterexample to the triangle inequality holds.
\item
Instead of dividing by $K(V_{\max})$ in \eqref{eq.prop}
divide by $K(\{V_{\max}\})$ where $\{V_{\max}\}$ is the set
of elements in $V_{\max}$. To equate the sets approximately 
with the corresponding lists,
 change $Z$ to $\{y_1,y_2\}$
where $y_i$ equals $y$ but with the $i$th bit flipped ($i=1,2$).
Again, the triangle inequality does not hold.
\item
Instead of dividing by $K(V')$ in \eqref{eq.prop}
divide by $K(\{V'\})$ where $\{V'\}$ is the set
of elements in $V'$.
Change $Z$ as in the previous item.
Again, the triangle inequality does not hold.
\end{itemize}

\section{Appendix: Kolmogorov Complexity Theory}\label{sect.kolm}

Theory and applications are given in the textbook \cite{LiVi08}.
Here we give some relations that are needed in the paper.
The {\em information about $x$ contained in $y$} is defined as
$I(y:x)=K(x)-K(x | y)$. A deep, and very useful, result
due to L.A. Levin and A.N. Kolmogorov \cite{ZL70} called
{\em symmetry of information} shows that
\begin{equation}\label{eq.soi}
K(x,y)=K(x)+K(y | x) = K(y)+K(x | y),
\end{equation}
with the equalities holding up to $\log K$  additive precision.
Here, $K=\max\{K(x),K(y)\}$.
Hence, up to an additive logarithmic  term $I(x:y) = I(y:x)$ and we
call this the {\em mutual (algorithmic) information} between $x$ and $y$.

The {\em universal a priori probability} of $x$ is
$Q_U(x) = \sum_{U(p)=x} 2^{-|p|}$.
The following results are due to L.A. Levin~\cite{Le74}.

There exists a lower semicomputable function 
${\bf m}: \{0,1\}^* \rightarrow [0,1]$ with $\sum_x {\bf m}(x) \leq 1$,
such that for every lower semicomputable 
function $P:  \{0,1\}^* \rightarrow [0,1]$ with $\sum_x P(x) \leq 1$
we have
\begin{equation}\label{eq.dominate}
2^{K(P)} {\bf m}(x) \geq P(x),
\end{equation}
 for every $x$.
Here $K(P)$ is the length of a shortest program for the reference
universal prefix Turing machine to lower semicompute the function $P$.
For every $x\in \{0,1\}^*$,
\begin{equation}\label{eq.coding}
K(x)= -\log Q_U (x)=-\log {\bf m}(x)
\end{equation}
with equality up to an additive constant independent of $x$.
Thus, the Kolmogorov complexity of a 
string $x$ coincides up to an additive constant 
term with the logarithm of $1/Q_U(x)$ and also
with the logarithm of $1/{\bf m}(x)$.
This result is called the ``Coding Theorem'' since it shows that the
shortest upper semicomputable code is a Shannon-Fano code of the
greatest lower semicomputable probability mass function. 

\bibliographystyle{plain}

\end{document}